\documentclass[english,10pt]{article}
\usepackage[papersize={19cm,25cm},body={15cm,21cm},centering]{geometry}
\usepackage[T1]{fontenc}

\usepackage{amsmath}
\usepackage{amssymb}
\usepackage{url}
\usepackage{babel}
 \usepackage{graphicx}
\usepackage{epsfig,subfigure,color}
\usepackage{graphics}
 \usepackage{graphicx}
\usepackage{amsmath,amsfonts,amsthm,latexsym}
\usepackage{mathrsfs,amsbsy}
\usepackage{multirow}
\usepackage{epstopdf}
\usepackage{algorithm}
\usepackage{algorithmic}
\usepackage{cite}

\newtheorem{theorem}{Theorem}
\newtheorem{lemma}{Lemma}

\theoremstyle{definition}

\newcommand{\tabincell}[2]{\begin{tabular}{@{}#1@{}}#2\end{tabular}}

\title{Tucker Decomposition Network: Expressive Power and Comparison}
\author{
	Ye LIU\\
	  Department of Mathematics\\
	  Hong Kong Baptist University \\
	  Hong Kong\\
	  \texttt{16482549@life.hkbu.edu.hk}
	  \and
Junjun Pan\thanks{Junjun Pan and Ye Liu contribute equally to this paper}\\
	Department of Mathematics\\ and Operational Research \\
 Universit\'e de Mons\\Belgium\\
	\texttt{junjun.pan@umons.ac.be}
	\and
Michael Ng\thanks{corresponding author}\\
Department of Mathematics\\
Hong Kong Baptist University \\
Hong Kong\\
\texttt{mng@math.hkbu.edu.hk} }
\date{}
\begin{document}
\maketitle

\begin{abstract}

Deep neural networks have achieved a great success in solving many machine learning and computer vision problems.
The main contribution of this paper is to develop a deep network based on Tucker tensor decomposition, and analyze its expressive power.
It is shown that the expressiveness of Tucker network is more powerful than that of shallow network.
In general, it is required to use an exponential number of nodes in a shallow network in order to represent a Tucker network.
Experimental results are also given to compare the performance of the proposed Tucker network with
hierarchical tensor network and shallow network, and demonstrate the usefulness of Tucker network in image
classification problems.
\end{abstract}

\section{Introduction}

Deep neural networks have achieved a great success in solving many practical problems. Deep learning methods are based on multiple levels of representation in learning. Each level involves simple but non-linear units for learning. Many deep learning networks have been developed and applied in various applications successfully. For example, convolutional neural networks (CNNs) \cite{he2016deep,lecun1995convolutional,szegedy2015going} have been well applied in computer vision problems, recurrent neural networks (RNNs) \cite{gers1999learning,graves2013speech,mikolov2011extensions} are used in audio and natural language processing. For more detailed discussions, see \cite{lecun2015deep} and its references.

In the recent years, more and more works focus on the theoretical explanations of neural networks. One important topic is the expressive power, i.e., comparing the expressive ability of different neural networks architectures. In the literature
\cite{delalleau2011shallow,eldan2016power,hastad1986almost,haastad1991power,martens2014expressive,montufar2014number,pascanu2013number,poggio2015theory,telgarsky2016benefits}, researches have been done in the investigation of the depth efficiency of neural networks. It is natural to claim that a deep network can be more
powerful in the expressiveness than a shallow network. Recently, Khrulkov et al. \cite{khrulkov2017expressive} applied a tensor train decomposition to exploit the expressive power of RNNs experimentally. In \cite{cohen2016expressive}, Cohen et al. theoretically analyzed
specific shallow convolutional network by using CP decomposition and specific deep convolutional network based on
hierarchical tensor decomposition. The result of the paper is that the expressive power of such deep convolutional networks
is significantly better than that of shallow networks. Cohen et al. in \cite{cohen2016convolutional} generalized convolutional arithmetic circuits into convolutional rectifier networks to handle activation functions, like ReLU. They showed that the depth efficiency of convolutional rectifier networks is weaker than that of convolutional arithmetic circuits.

Although many attempts in theoretical analysis success, the understanding of expressiveness is still needed to be developed. The main contribution of this paper is that a new deep network based on Tucker tensor decomposition is proposed. We analyze the expressive power of the new network, and show that it is required to use an exponential number of nodes in a shallow network  to represent a Tucker network. Moreover, we compare the performance of the proposed Tucker network, hierarchical tensor network and shallow network on two datasets (Mnist and CIFAR) and demonstrate that the proposed Tucker network outperforms the other two networks.

The rest of this paper is organized as follows. In Section 2, we briefly review tensor decompositions. We present the proposed Tucker network and show its expressive power in Section 3. In Section 4, experimental results are presented to demonstrate the performance of the Tucker network. Some concluding remarks are given in Section 5.

\section{Tensor Decomposition}

A $N$-dimensional tensor $\mathcal{A}$ is a multidimensional array, i.e., $\mathcal{A}\in \mathbb{R}^{M_1\times M_2\times \cdots\times M_N}$. Its $i$-th unfolding matrix is defined as $A_{(i)}\in \mathbb{R}^{M_i\times M_{i+1}\cdots M_{N}M_{1}\cdots M_{i-1}}$.  Given an index subset $\mathbf{p}=\{p_1,p_2,\cdots,p_{n_1}\}$ and the corresponding compliment set $\mathbf{q}=\{q_1,q_2,\cdots, q_{ n_2}\}$, $n_1+n_2=N$, the $(p,q)$-matricization of $\mathcal{A}$ is denoted as a matrix $[\mathcal{A}]^{(p,q)}\in \mathbb{R}^{M_{p_{1}} \cdots M_{p_{n_{1}}}\times M_{q_1}\cdots M_{q_{ n_2}}}$, obtained by reshaping tensor $\mathcal{A}$ into matrix.

We also introduce two important operators in tensor analysis, tensor product and Kronecker product. Given tensors $\mathcal{A}$ and $\mathcal{B}$ of order $N_1$ and $N_2$ respectively, the  tensor product is defined as $(\mathcal{A}\circ\mathcal{B})_{d_1,\cdots,d_{N_1+N_2}}=\mathcal{A}_{d_1,\cdots,d_{N_1}}\mathcal{B}_{d_{N_1+1},\cdots,d_{N_1+N_2}}$. Note that when $N_1=N_2=1$, the tensor product is the outer product of vectors. $\otimes $ denotes Kronecker product which is an operation on two matrices, i.e., for matrices $\mathbf{A}\in \mathbb{R}^{m_1\times n_1}$, $\mathbf{B}\in \mathbb{R}^{m_2\times n_2}$, $\mathbf{A}\otimes \mathbf{B}\in \mathbb{R}^{m_1m_2\times n_1n_2}$, defined by $(\mathbf{A}\otimes \mathbf{B})_{m_2(r-1)+v,n_2(s-1)+w}=a_{rs}b_{vw}$.

Moreover, we use $[k]$ to denote the set $\{1,\cdots, k\}$ for simplicity.

In the following,  we review some well-known tensor decomposition methods and related convolutional networks.

\textbf{CP decomposition:} \cite{carroll1970analysis,harshman1970foundations} Given a tensor $\mathcal{A}\in \mathbb{R}^{M_1\times M_2\times \cdots\times M_N}$, the CANDECOMP/PARAFAC decomposition (CP) is defined as follows:
\begin{equation}\label{CP}
\mathcal{A}=\sum^Z_{z=1}\lambda_z \mathbf{a}^{z,1}\circ \mathbf{a}^{z,2} \circ \cdots \mathbf{a}^{z,N}, \ {\rm i.e.}, \
A_{d_1,d_2,\cdots,d_N}=\sum^Z_{z=1}\lambda_z \mathbf{a}^{z,1}_{d_1} \mathbf{a}^{z,2}_{d_2} \cdots \mathbf{a}^{z,N}_{d_N},
\end{equation}
where $\mathbf{\lambda}^y\in \mathbb{R}^Z$, $\mathbf{a}^{z,i}\in \mathbb{R}^{M_i}$. The minimal value of $Z$ such that CP decomposition exists is called the CP rank of $\mathcal{A}$ denoted as $rank_{CP}(\mathcal{A})=Z$.


\textbf{Tucker decomposition:}\cite{de2000multilinear, tucker1966some}
Given a tensor $\mathcal{A}\in \mathbb{R}^{M_1\times M_2\times \cdots\times M_N}$, the Tucker decomposition is defined as follows:
\begin{equation*}
\mathcal{A}=\mathcal{G}\times_1 \mathbf{U}^{(1)}\times_2 \mathbf{U}^{(2)}\times\cdots\times_N \mathbf{U}^{(N)}, \ {\rm i.e.}, \ A_{d_1, \cdots,d_N}=\sum_{j_1,\cdots,j_N}g_{j_1,\cdots,j_N}U^{(1)}_{d_1,j_1}\cdots U^{(N)}_{d_N,j_N},
\end{equation*}
which can be written as,
\begin{equation}\label{e1}
\mathcal{A}=\sum_{j_1,j_2,\cdots,j_N}g^y_{j_1,j_2,\cdots,j_N}(\mathbf{u}^{(1)}_{j_1}\circ \mathbf{u}^{(2)}_{j_2}\circ \cdots\circ \mathbf{u}^{(N)}_{j_N}),
\end{equation}
where
$\mathcal{G}=(g_{j_1,j_2,\cdots, j_N})\in \mathbb{R}^{J_1\times J_2\times\cdots\times J_N}$, $\mathbf{u}^{(1)}_{j_1}\in \mathbb{R}^{M}$ , $\cdots$, $\mathbf{u}^{(N)}_{j_N}\in \mathbb{R}^{M}$, $j_1\in [J_1]$, $\cdots$, $j_N\in [J_N]$. The minimal value of $(J_1,J_2,\cdots,J_N)$ such that (\ref{e1}) holds is called Tucker rank of $\mathcal{A}$, denoted as $rank_{Tucker}(\mathcal{A})=(J_1,J_2,\cdots,J_N)$. If $rank_{Tucker}(\mathcal{A})=(J,\cdots J)$, we simplicity denoted as $rank_{Tucker}(\mathcal{A})=J$.

\textbf{HT decomposition:} The Hierarchical Tucker (HT) Tensor format is a multilevel variant of a tensor decomposition format. The
definition requires the introduction of a tree. For detailed discussion, see
\cite{grasedyck2010hierarchical, grasedyck2011introduction,hackbusch2012tensor}.
Given a tensor $\mathcal{A}\in \mathbb{R}^{M_1\times M_2\times \cdots\times M_N}$, $N=2^L$. The $2^L$ hierarchical tensor decomposition has the following form:
\begin{eqnarray}\label{HT}
\mathcal{A} & = & \sum^{r_{L-1}}_{\alpha=1}a^{L}_\alpha \phi^{L-1,1,\alpha}\circ \phi^{L-1,2,\alpha},\\
\phi^{l,t_l,\gamma_l} & = & \sum^{r_{l-1}}_{\alpha=1}a^{l,t_l,\gamma_l}_\alpha \phi^{l-1,2t_l-1,\alpha}\circ \phi^{l-1,2t_l,\alpha},
\quad t_l\in [2^{L-l}],\gamma_l\in[r_{l}], l=L-1,L-2,\cdots,2. \nonumber \\
\phi^{1,t_1,\gamma_1} & = & \sum^{r_0}_{\alpha=1}a^{1,t_1,\gamma_1}_\alpha \mathbf{u}^{0,2t_1-1,\alpha}  \circ\mathbf{u}^{0,2t_1,\alpha},
\quad t_1\in [2^{L-1}],\gamma_1\in[r_{1}].  \nonumber
\end{eqnarray}
where $\{\mathbf{u}^{0,t_1,\alpha}\}_{t_1\in [N],\alpha\in [r_0]}$ are the generated vectors of tensor $\mathcal{A}$. $r_l$ refer to level-$l$ rank. We denote $rank_{HT}(\mathcal{A})=(r_0,r_1,\cdots,r_{l-1})$. If all the ranks are equal to $r$,  $rank_{HT}(\mathcal{A})=r$ for simple.

\subsection{Convolutional Networks}
\begin{figure}
	\centering
	\includegraphics[width=13cm,height=4.5cm]{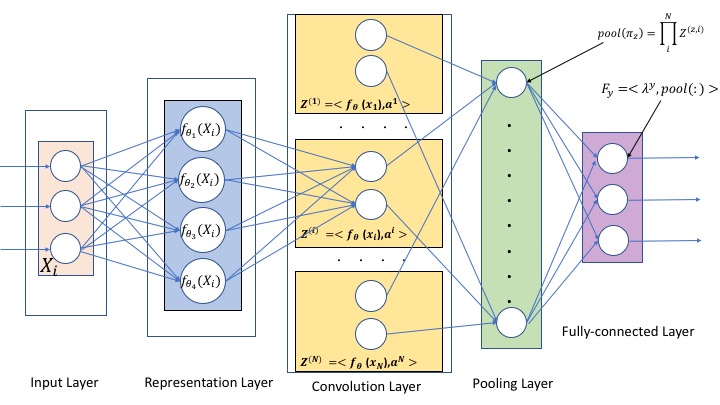}
	
	\caption{A CP network.}
	
\label{CP-network}	
\end{figure}
\begin{figure}
	\centering
    \includegraphics[width=13cm,height=4.5cm]{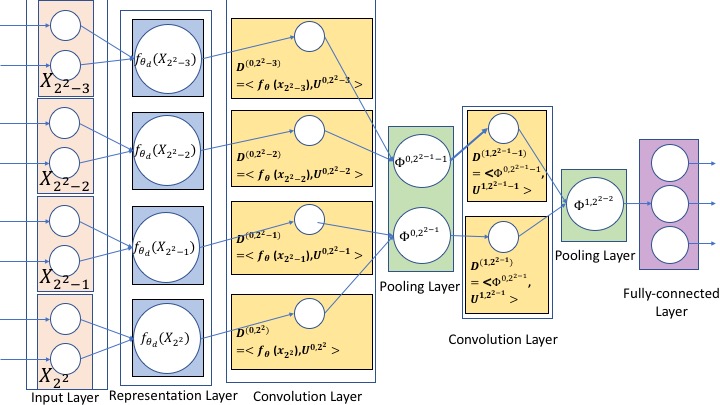}

	\caption{A example of HT network with $L=2$.}
	
	\label{HT-network}
\end{figure}
Given a dataset of pairs $\{(X^{(b)},y^{(b)})\}^D_{b=1}$, each object $X^{(b)}$ is represented as a set of vectors $X^{(b)}=(\mathbf{x}_1,\mathbf{x}_2,\cdots,\mathbf{x}_N)$  with ${\bf x}_i\in \mathbb{R}^s$.
By applying parameter dependent functions $\{f_{\theta_d}: \mathbb{R}^s\rightarrow \mathbb{R}\}^M_{d=1}$, we construct a representation map $f_\theta: \mathbb{R}^s\rightarrow \mathbb{R}^M$.  Object $X=(\mathbf{x}_1,\mathbf{x}_2,\cdots,\mathbf{x}_N)$ with $\mathbf{x}_i\in \mathbb{R}^s$ is classified into one of categories $\mathfrak{Y}=\{1,2,\cdots,Y\}$. Classification is carried out through the maximization of the following score function:
\begin{equation}\label{score}
F_y(\mathbf{x}_1,\mathbf{x}_2,\cdots,\mathbf{x}_N)=\sum^M_{d_1,d_2,\cdots,d_N} \mathcal{A}^y_{d_1,d_2,\cdots,d_N}\prod^N_{i=1}f_{\theta_{d_{i}}}(\mathbf{x}_i),
\end{equation}
where $\mathcal{A}^y\in \mathbb{R}^{M\times M\times \cdots\times M}$ is a trainable coefficient tensor.

The representation functions  $\{f_{\theta_d}: \mathbb{R}^s\rightarrow \mathbb{R}\}^M_{d=1}$  have many choices. For example, neurons-type functions  $f_{\theta_d}(\mathbf{x})= \sigma( \mathbf{x}^T\mathbf{w}_d+\mathbf{b}_d)$ for parameters  $\theta_d=(\mathbf{w}_d,\mathbf{b}_d)$ and  point-wise non-linear activation $\sigma(\cdot)$. We list some commonly used activation functions here, for example hard threshold: $\sigma(z)=1$ for $z>0$,  otherwise 0;  the rectified linear unit (ReLU) $\sigma(z)=\max\{z,0\}$; and sigmoid $\sigma(z)=\frac{1}{1+e^{-z}}$.


The main task is to estimate the parameters $\theta_1,\theta_2,\cdots,\theta_M$ and the coefficient tensors $\mathcal{A}^1, \cdots, \mathcal{A}^Y$. The computational challenge is that the coefficient tensor has an exponential number of entries. We can utilize tensor decompositions to address this issue.

If the coefficient tensor is in CP decomposition,  the network corresponding to CP decomposition is called shallow
network(or CP Network), see Figure \ref{CP-network}. We obtain its score function:
\begin{equation}\label{CPnet}
	F_y=\sum^Z_{z=1}\lambda^y_z \prod^N_{i=1}\sum^M_{d=1}\mathbf{a}^{z,i}_df_{\theta_d}(\mathbf{x^i}).
	\end{equation}
Note that the same vectors $\mathbf{a}^{z,i}$ are shared across all classes $y$. If set $Z=M^N$, the model is  universal, i.e., any tensors $\mathcal{A}^1,\cdots, \mathcal{A}_Y$ can be represented.

If the coefficient tensors are in  HT format like (\ref{HT}), the network refer to HT network. An example of HT network with $L=2$ is showed in Figure \ref{HT-network}. Cohen et al. \cite{cohen2016expressive} analyzed the expressive power of HT network and proved that  a shallow network with exponentially large width is required to emulate a HT network.

\section{Tucker Network}
In this section, we propose a Tucker network. If the coefficient tensors in (\ref{score}) are  in Tucker format (\ref{e1}), we refer it as Tucker network, i.e.,
\begin{equation}\label{tuck}
\mathcal{A}^y=\sum_{j_1,j_2,\cdots,j_N}g^y_{j_1,j_2,\cdots,j_N}(\mathbf{u}^{(1)}_{j_1}\circ \mathbf{u}^{(2)}_{j_2}\circ \cdots\circ \mathbf{u}^{(N)}_{j_N}).
\end{equation}
Suppose $J_k=\max_{y} rank(\mathbf{A}^y_{(k)})$ for same vectors $\mathbf{u}^{(k)}_{j_k}$($k\in [N]$) in (\ref{tuck}).
Here $\mathbf{A}^y_{(k)}$ be the $k$-th unfolding of tensor $\mathcal{A}^y$. If set $J_1=J_2=\cdots=J_N=J$, the number of parameter is: $YJ^N+MNJ$. If set $J_1=J_2=\cdots=J_N=M$, the model is universe, any tensor  can be represented by Tucker format,  number of $YM^N+M^2N$ parameters are needed.  Note that the score function for Tucker network:
\begin{eqnarray*}
	F_y
&=&
\sum^M_{d_1,\cdots, d_N}\left (\sum_{j_1,\cdot,j_N}g^y_{j_1,\cdots,j_N}u^{(1)}_{d_1,j_1}\cdots u^{(N)}_{d_N,j_N} \right )\prod^N_{i=1}f_{\theta_{d_i}}(\mathbf{x}_i)\\
	&=&
\sum_{j_1,\cdot,j_N}g^y_{j_1,\cdots,j_N}\prod^N_{i=1}\left (\sum^M_{d=1}u^{(i)}_{d,j_i}f_{\theta_{d}}(\mathbf{x}_i) \right).
\end{eqnarray*}

The Tucker network architecture is given in Figure \ref{Tuckernet}. The outputs from convolution layer are

$$
\mathbf{v}^{(i)}=U^{(i)T}f_{\theta}(\mathbf{x}_i),~~i\in [N],
$$

where $U^{(i)}=(u^{(i)}_{d,j_i})_{1\leq d\leq M; 1\leq j_i\leq J_i}$, $i\in [N]$.
The last output, i.e., score value is given as follows:

$$
F_y=\langle \mathcal{G}^y, \mathbf{v}^{(1)}\circ \mathbf{v}^{(2)}\cdots \circ \mathbf{v}^{(N)}\rangle,
$$
where $\langle \cdot \rangle$ is tensor scalar product, i.e., the sum of entry-wise product of two tensors.
Because $\mathcal{G}^y$ is a $N$ order tensor of smaller dimension $J_1\times \cdots \times J_N$, it can be further decomposed with a deeper network.  In this sense, Tucker network is also a kind of deep network.

\begin{figure}
	\centering
	\includegraphics[width=13cm,height=5cm]{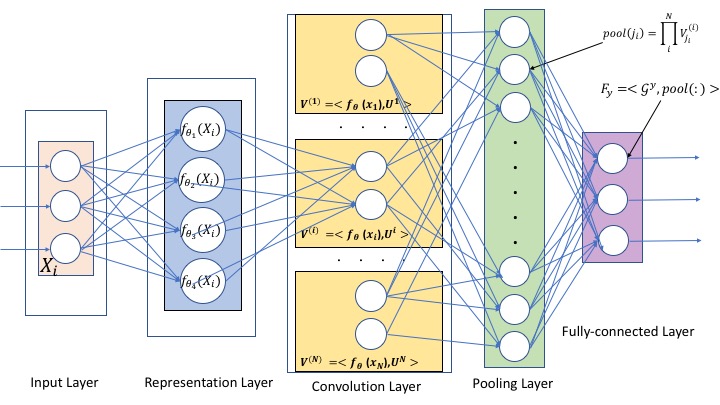}
	
	\caption{Tucker network}
	
	\label{Tuckernet}
\end{figure}
The following theorem demonstrates the expressive power of  Tucker network.

\begin{theorem}
	Let $\mathcal{A}^y$ be a tensor of  order $N$ and dimension $M$ in each mode, generated by Tucker form in (\ref{tuck}). Define $r=\max_{p,q}[\mathcal{G}]^{(p,q)}$ for all possible subsets $\mathbf{p},\mathbf{q}$, consider the space of all possible configurations for parameters. In the space, $\mathcal{A}^y$ will have CP rank of at least $r$ almost everywhere, i.e.,the Lebesgue measure of the space whose CP rank is less than $r$ is zero.
\end{theorem}

The proof can be found in the supplementary section.
We remark that
if $J_1=J_2=\cdots=J_N=J$, when $N$ is even, the Lebesgue measure of the Tucker format space whose CP rank is less than $J^{\frac{N}{2}}$ is zero;   when $N$ is odd, the Lebesgue measure of the Tucker format space whose CP rank is less than $J^{\frac{N-1}{2}}$ is also zero.

\subsection{Connection with HT Network}

In this subsection, to compare the expressive power of HT and Tucker network, we discuss the relationship between Tucker format and  hierarchical Tucker tensor format firstly. Here we only consider $N=2^L$ hierarchical tensor format, its corresponding HT network (\ref{HT}) has been well discussed in \cite{cohen2016expressive}.

We start it from $L=2$ hierarchical Tucker tensor, its HT network architecture is shown in Figure \ref{HT-network} . Given a $2^2$ order tensor, its hierarchical tensor format can always be written as
\begin{eqnarray*}
   \mathcal{A}   &=  & \sum^{r_{1}}_{\alpha=1}a^{2}_\alpha \phi^{1,1,\alpha}\circ \phi^{1,2,\alpha}, \\
   \phi^{1,1, \gamma_1} &= & \sum^{r_0}_{\alpha=1}a^{1,1,\gamma_1}\mathbf{u}^{0,1,\alpha}\circ \mathbf{u}^{0,2,\alpha},\\
    \phi^{1,2, \gamma_1} &=& \sum^{r_0}_{\alpha=1}a^{1,2,\gamma_1}\mathbf{u}^{0,3,\alpha}\circ \mathbf{u}^{0,4,\alpha},\quad \gamma_1 \in [r_1].
\end{eqnarray*}
$\mathbf{u}^{0,i,\alpha}, i\in[4]$ are vectors size of $M$. Here we suppose that $M\geq r_0$. Denote $\mathbf{A}=[\mathcal{A}]^{(12,34)}$,
we have $\mathbf{A}=\Phi^{1,1}\Sigma_1(\Phi^{1,2})^T$,
where $$\Phi^{1,1}= \left(
                     \begin{array}{ccc}
                       vec(\phi^{1,1,1}) & \cdots &  vec(\phi^{1,1,r_1}) \\
                     \end{array}
                   \right),$$  $$\Phi^{1,2}= \left(
                     \begin{array}{ccc}
                       vec(\phi^{1,2,1}) & \cdots &  vec(\phi^{1,2,r_1})\\
                     \end{array}
                   \right),$$ $$\Sigma_1=diag(a^2_1,\cdots,a^2_{r_1}).$$
$vec(\cdot)$ is a linear transformation that converts a matrix into a column vector. $diag(\cdot)$ is diagonal operator that transform a vector into a diagonal matrix.  Similarly, we have,
\begin{equation*}
  \phi^{1,1,\gamma_1}=\mathbf{U}^{0,1}\Sigma ^{0,1,\gamma_1}(\mathbf{U}^{0,2})^T,\quad \phi^{1,2,\gamma_1}=\mathbf{U}^{0,3}\Sigma ^{0,2,\gamma_1}(\mathbf{U}^{0,4})^T,
\end{equation*}
where $\mathbf{U}^{0,i}=\left(
                     \begin{array}{ccc}
                       \mathbf{u}^{0,i,1} & \cdots &   \mathbf{u}^{0,i,r_0} \\
                     \end{array}
                   \right)$, $i=[4]$, and
 $\Sigma^{0,j,\gamma_1}=diag(a^{1,j,\gamma_1}_1,\cdots,a^{1,j,\gamma_1}_{r_0})$, $j=[2], \gamma_1\in [r_1]$.

From the property of Kronecker product: $\mathbf{Y}=\mathbf{A}\mathbf{X}\mathbf{B} \Leftrightarrow vec(\mathbf{Y})=(\mathbf{B}^T\otimes \mathbf{A})vec(\mathbf{X})$, we deduce that,
\begin{equation*}
  vec(\phi^{1,1,\gamma_1})=(\mathbf{U}^{0,2}\otimes \mathbf{U}^{0,1})vec(\Sigma ^{0,1,\gamma_1}),\quad vec(\phi^{1,2,\gamma_1})=(\mathbf{U}^{0,4}\otimes \mathbf{U}^{0,3})vec(\Sigma ^{0,2,\gamma_1}),\quad \gamma_1\in [r_1].
\end{equation*}
Therefore,
\begin{equation*}
\Phi^{1,1}=(\mathbf{U}^{0,2}\otimes \mathbf{U}^{0,1})\mathbf{D}_1,\quad \Phi^{1,2}=(\mathbf{U}^{0,4}\otimes \mathbf{U}^{0,3})\mathbf{D}_2,
\end{equation*}
with $\mathbf{D}_1=\left(
                     \begin{array}{ccc}
                       vec(\Sigma^ {0,1,1} )& \cdots& vec(\Sigma^{0, 1, r_1}) \\
                     \end{array}
                   \right)
$, $\mathbf{D}_2=\left(
                     \begin{array}{ccc}
                       vec(\Sigma ^{0,2, 1})& \cdots& vec(\Sigma^{0,2, r_1}) \\
                     \end{array}
                   \right).
$
We can get that
\begin{equation}\label{matlayer2}
   \mathbf{A}=(\mathbf{U}^{0,2}\otimes \mathbf{U}^{0,1})\mathbf{D}_1\Sigma_1\mathbf{D}^T_2(\mathbf{U}^{0,4}\otimes \mathbf{U}^{0,3})^T.
\end{equation}
Therefore,
\begin{equation*}
\mathcal{A}=\mathcal{G}\times_1 \mathbf{U}^{0,2}\times_2 \mathbf{U}^{0,1}\times_3 \mathbf{U}^{0,3}\times_4 \mathbf{U}^{0,4},
\end{equation*}
with $\mathcal{G}^{(12,34)}=D_1\Sigma_1 D^T_2$.  It implies that a $2^2$ hierarchical tensor format can be written as a $2^2$ order Tucker tensor.
Worth to say, from (\ref{matlayer2}), the rank of $\mathbf{A}$ is less than that of its factor matrices. Because of the structure of $\mathbf{D}_1$, $\mathbf{D}_2$, we get that $rank(\mathbf{D}_1)\leq \min\{r_0,r_1\}$ and also $rank(\mathbf{D}_2)\leq \min\{r_0,r_1\}$. From the rank property, $rank(\mathbf{A})\leq \min \{r_0,r_1\}$.

When the hierarchical tensor has $L$ layers, we can similarly deduced the following results.

\begin{theorem}\label{ThmofTucker-HT}
	Any  $2^L$ hierarchical tensor can be represented as a  $2^L$ order Tucker tensor and vice versa.
\end{theorem}

\begin{theorem}\label{Thmrank}
	For any tensor $\mathcal{A}$, if $rank_{HT}\leq r$, then $rank_{Tucker}\leq r$.
\end{theorem}

The detailed proofs of Theorem \ref{ThmofTucker-HT} and Theorem \ref{Thmrank} can be found in Appendix.

According to Theorem \ref{Thmrank}, given a hierarchical Tucker network of width $r$, we know that the width of Tucker network is not possible  larger than $r$.

\section{Experimental Results}

We designed experiments to compare the performance of three networks: Tucker network, HT network and shallow network. The results illustrate the usefulness of Tucker network.
We implement shallow network, Tucker network and HT network with \texttt{TensorFlow}\cite{abadi2016tensorflow} back-end, and test three networks on two different data sets: Mnist \cite{lecun1990handwritten} and CIFAR-10 \cite{krizhevsky2009learning}. All three networks are trained by using the back-propagation algorithm. In all three networks, we choose ReLU as the activation function in the \textit{representation layer} and apply batch normalization \cite{ioffe2015batch} between \textit{convolution layer} and \textit{pooling layer} to eliminate numerical overflow and underflow.


We choose Neurons-type $f_{\theta}(\mathbf{x})= \sigma( \mathbf{x}^T\mathbf{w}+\mathbf{b})$ with ReLU nonlinear activation $\sigma$ as representation map $\{f_{\theta_d}$: $\mathbb{R}^s\rightarrow \mathbb{R}\}_{d=1}^M$. Actually the representation mapping now is acted as a convolution layer in general CNNs. Each image patch is transformed through a representation function with parameter sharing across all the image patches.
\textit{Convolution layer} $\mathbf{v}^{(i)}=U^{(i)T}f_{\theta}(\mathbf{x}_i),i\in [N]$ in Figure \ref{Tuckernet} actually can been seen as a locally connected layer in CNN. It is a specific convolution layer without parameter sharing, which means that the parameters of filter would differ when sliding across different spatial positions. In the hidden layer, without overlapping, a 3D convolution operator size of $1\times1\times j_i$  is applied.
Following is a product \textit{pooling layer} to realize the outer product computation  $\prod^N_{i=1}\mathbf{v}^{(i)}_{j_i}$. It can be explained as a pooling layer with local connectivity property, which only connects a neuron with partial neurons in the previous layer. The output of neuron is the multiplication of entries in the neurons connected to it. The \textit{fully-connected layer} simply apply the linear mapping on the output of pooling layer. The output of Tucker network would be a vector $Y$ corresponding to class scores.

\subsection{Mnist}

The MNIST database of handwritten digits has a training set of 60000 examples, and a test set of 10000 examples with 10 categories from 0 to 9. Each image is of $28\times 28$ pixels.
In the experiment, we select the gradient descent optimizer for back-propagation with batch size 200, and use a exponential decay learning rate with 0.2 initial learning rate, 6000 decay step and 0.1 decay rate.
Figure \ref{Mnist_acc} shows the training and test accuracy of three networks with 3834 number of parameters that have been learned. The parameters contains four parameters in batch normalization (mean, std, alpha, beta). We list filter size, strides size and rank as well in Table \ref{mnist_sensitive_par}. It is obvious that Tucker network outperforms shallow network and HT network. Moreover, we test the sensitivity of Tucker network with the change of rank, and compare the performance with the other two networks with the same number of parameters. Figure \ref{mnist_sensitive}  illustrates the sensitivity performance, each value records the highest accuracy in training or test data. Tucker network can achieve the highest accuracy at most times.

\subsection{CIFAR-10}
CIFAR-10 data \cite{krizhevsky2009learning} is a more complicated data set consisting of 60000 color images size of $32\times 32$ with 10 classes. Here, we use the gradient descent optimizer with 0.05 learning rate and 200 batch size to train.
 In Figure \ref{cifar10_acc} we report the training and test accuracy with 23790 trained parameters. Table \ref{cifar10_sensitivity_par} shows the parameter details of sensitivity test, whose results are displayed in Figure \ref{cifar10_sensitive} . From Figure \ref{cifar10_acc} and Figure \ref{cifar10_sensitive} , Tucker network still has more excellent performance when fitting a more complicated data set.

\section{Conclusion}

In this paper, we presented a Tucker network and prove the expressive power theorem. We stated that a shallow network of exponentially large width is required to mimic Tucker network. A connection between Tucker network and HT network is discussed. The experiments on Mnist and CIFAR-10 data show the usefulness of our proposed Tucker network.

\begin{figure}
	\begin{minipage}[t]{0.5\linewidth}
		\centering
		\includegraphics[height=1.8in,width=2.8in]{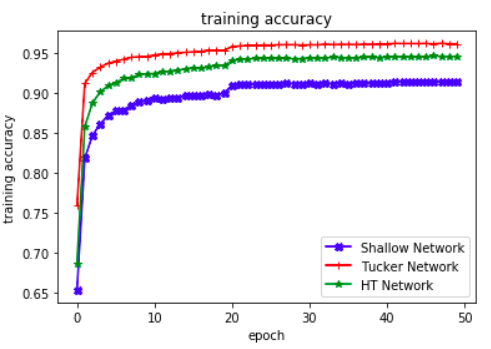}
	\end{minipage}
	\begin{minipage}[t]{0.5\linewidth}
		\centering
		\includegraphics[height=1.8in,width=2.8in]{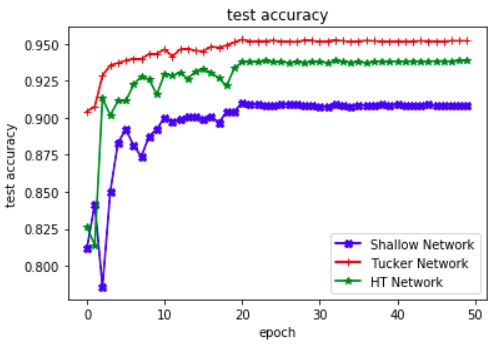}
	\end{minipage}

	\caption{Training(left) and testing(right) accuracy of Tucker network for Mnist data.}
	
	\label{Mnist_acc}
\end{figure}

\begin{figure}
	\begin{minipage}[t]{0.5\linewidth}
		\centering
		\includegraphics[height=1.8in,width=2.6in]{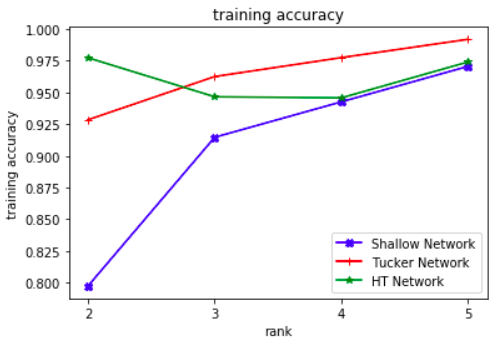}
	\end{minipage}
	\begin{minipage}[t]{0.5\linewidth}
		\centering
		\includegraphics[height=1.8in,width=2.6in]{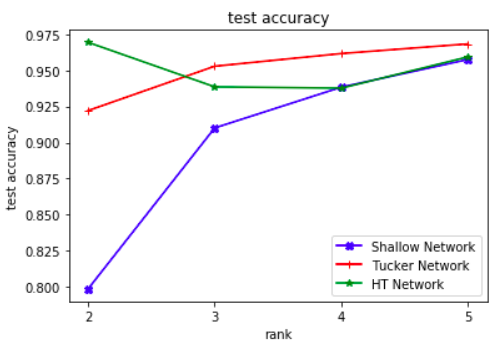}
	\end{minipage}

	\caption{The performance of Tucker network, Shallow network and Hierarchical tensor network with the change of rank in Mnist data: training accuracy(left); testing accuracy(right).}
	
	\label{mnist_sensitive}
\end{figure}

\begin{table}
{
	\centering
	
	\caption{The parameters setting of Fig. \ref{mnist_sensitive} in Tucker network, Shallow network and Hierarchical tensor network.}
	
	\label{mnist_sensitive_par}
	\begin{tabular}{|c|c|c|c|c|c|}
		\hline
		Network & \tabincell{c}{Num of\\parameters} & \tabincell{c}{Num $M$ of representation \\function} & Filter size & Strides size& \tabincell{c}{Rank \\$J_i(i\in [N])$} \\
		\hline
		Tucker  &  & 10& 14 $\times$ 23 & 14 $\times$ 5& {\color{red}2} \\
		HT  & 3478& 14 & 14 $\times$14 & 14 $\times$ 14& 8\\
		Shallow  & & 10 & 16 $\times$21& 12 $\times$ 7&2\\
		\hline
		Tucker  &  & 12& 14 $\times$ 17 & 14 $\times$ 11& {\color{red}3} \\
		HT  & 3834& 18 & 14 $\times$14 & 14 $\times$ 14& 3\\
		Shallow  & & 16 & 14 $\times$16& 14 $\times$ 12&3\\
		\hline
		Tucker  &  & 12& 14 $\times$ 15 & 14 $\times$ 13& {\color{red}4} \\
		HT  & 5300& 12 & 16 $\times$26 & 12 $\times$ 2& 4\\
		Shallow  & & 10 & 20 $\times$21& 8 $\times$ 7&4\\
		\hline
		Tucker  &  & 11& 14 $\times$ 14 &  14$\times$ 14& {\color{red}5} \\
		HT  & 8657& 11 & 26 $\times$27 & 2 $\times$ 1& 11\\
		Shallow  & & 17 & 20 $\times$23& 8 $\times$ 5&10\\
		\hline
	\end{tabular}
}
\end{table}

\begin{figure}
	\begin{minipage}[t]{0.5\linewidth}
		\centering
		\includegraphics[height=1.8in,width=2.8in]{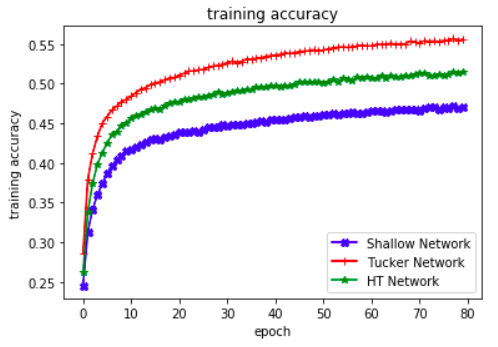}
	\end{minipage}
	\begin{minipage}[t]{0.5\linewidth}
		\centering
		\includegraphics[height=1.8in,width=2.8in]{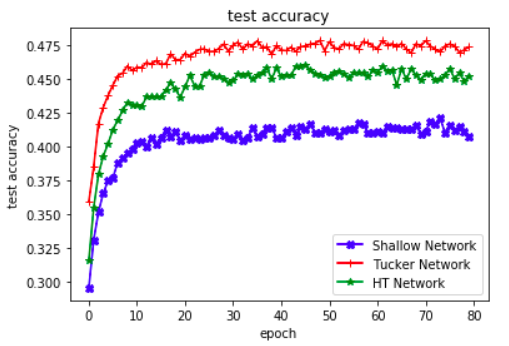}
	\end{minipage}

	\caption{Training(left) and testing(right) accuracy of CIFAR-10 data.}
	
	\label{cifar10_acc}
\end{figure}

\begin{figure}
	\begin{minipage}[t]{0.5\linewidth}
		\centering
		\includegraphics[height=1.8in,width=2.6in]{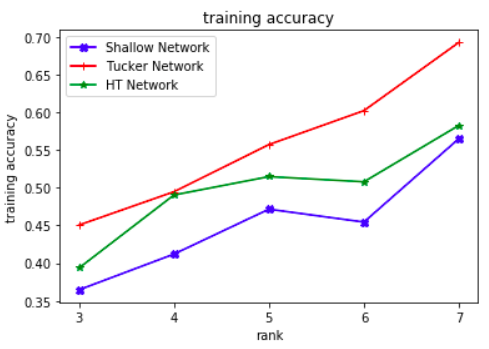}
	\end{minipage}
	\begin{minipage}[t]{0.5\linewidth}
		\centering
		\includegraphics[height=1.8in,width=2.6in]{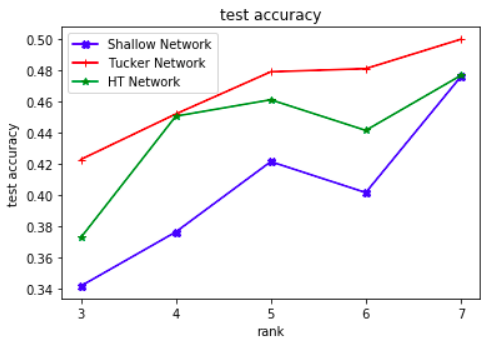}
	\end{minipage}

	\caption{The performance of Tucker network, Shallow network and HT network with the change of rank in CIFAR-10 data: training accuracy(left); testing accuracy(right).}
	
	\label{cifar10_sensitive}
\end{figure}

\begin{table}
{
	\centering
	
		\caption{The parameters setting of Fig. \ref{cifar10_sensitive} in Tucker network, Shallow network and Hierarchical tensor network.}
		
			\label{cifar10_sensitivity_par}
	\begin{tabular}{|c|c|c|c|c|c|}
		\hline
		Network & \tabincell{c}{Num of \\parameters} & \tabincell{c}{Num $M$ of \\representation ~function}& Filter size & Strides size&\tabincell{c}{Rank \\$J_i(i\in[N])$ } \\
		\hline
		Tucker  &  & 10& 16 $\times$ 26 & 16 $\times$ 6& {\color{red}3} \\
		HT  & 13432&  10&  21$\times$21 & 11 $\times$11&3\\
		Shallow  & & 10 &17  $\times$26&15 $\times$6&3\\
		\hline
		Tucker  &  & 10& 16 $\times$ 31&16 $\times$1 & {\color{red}4} \\
		HT  & 17626&  22&16  $\times$16&16$\times$16 &6 \\
		Shallow  & & 10 &20  $\times$29&12$\times$3&4\\
		\hline
		Tucker  &  & 20& 16 $\times$ 18 &16$\times$14&{\color{red}5} \\
		HT  & 23970& 30 & 16 $\times$16 &16$\times$16& 6\\
		Shallow  & & 12 & 25 $\times$26&7$\times$6&19\\
		\hline
		Tucker  &  & 24& 16 $\times$16  & 16$\times$16&{\color{red}6} \\
		HT  & 32016&  12&28  $\times$31 & 4$\times$1&9\\
		Shallow  & &  20& 17 $\times$31&15$\times$1&4\\
		\hline
		Tucker  &  & 31& 16 $\times$ 17 &16$\times$15 &{\color{red}7} \\
		HT  & 50233& 43 & 18 $\times$ 21& 14$\times$11&7\\
		Shallow  & & 37 & 17 $\times$26&15$\times$6&7\\
		\hline
	\end{tabular}
}
\end{table}

\bibliographystyle{abbrv}
\bibliography{Tuckernet_preprint}

\newpage

\appendix

\section{Appendix A. Proofs}
\subsection{Proof of Theorem 1}
In section 3, we presented Tucker network and showed its expressive power. To prove Theorem 1, we firstly state and prove three lemmas which will be needed for the proofs.
\begin{lemma}
For any $(p,q)$ matricization of a tensor $\mathcal{A}$ whose CP rank is $Z$,
$$
rank[\mathcal{A}]^{(p,q)}\leq Z.
$$
\end{lemma}
\begin{proof}
 \begin{eqnarray*}
	 rank[\sum^Z_{z=1}\lambda^y_z \mathbf{a}^{z,1}\circ\mathbf{a}^{z,2} \circ\cdots \circ \mathbf{a}^{z,N}]^{(p,q)}
&=&rank \sum^Z_{z=1}\lambda^y_z [\mathbf{a}^{z,1}\circ \mathbf{a}^{z,2} \circ \cdots \circ\mathbf{a}^{z,N}]^{(p,q)}\\
		& \leq & \sum^Z_{z=1} rank[\mathbf{a}^{z,1}\circ \mathbf{a}^{z,2}\circ\cdots \circ \mathbf{a}^{z,N}]^{(p,q)}=Z.
	\end{eqnarray*}
\end{proof}
In Lemma 1, we give the lower bound of the CP-rank. If the matricization of a tensor $\mathcal{A}$ has matrix rank $Z$, then using the above lemma, we get that the CP rank of $\mathcal{A}$ is larger than $Z$.

For a $N$ order tensor  $\mathcal{A}$  who is in Tucker format, its matricization has the following form.
\begin{lemma}
	Given a $N$ order tensor $\mathcal{A}\in \mathcal{R}^{M\times M\times\cdots\times M}$ whose Tucker format is $\mathcal{A}=\mathcal{G}\times_1\mathbf{U}^{(1)}\times \cdots\times_{N} \mathbf{U}^{(N)}$, index subsets $\mathbf{p}=\{p_1,\cdots,p_{n_1}\}$ and $\mathbf{q}=\{q_1,\cdots,q_{n_2}\}$, then
	$$[A]^{(p,q)}=(\mathbf{U}^{(p_1)}\otimes \mathbf{U}^{(p_2)} \otimes \cdots\otimes \mathbf{U}^{(p_{n_1})})[\mathcal{G}]^{(p,q)}(\mathbf{U}^{(q_{n_2})}\otimes \mathbf{U}^{(q_{{n_2}-1})} \otimes \cdots\otimes \mathbf{U}^{(q_1)})^T.$$
\end{lemma}
\begin{proof}
	$$
	A_{d_1,\cdots,d_N}=\sum_{j_1,\cdots,j_N}g_{j_1,\cdots,j_N}U^{(1)}_{d_1,j_1}\cdots U^{(N)}_{d_N,j_N}.
	$$
	Therefore,
	\begin{eqnarray*}
&A_{d_{p_1},\cdots,d_{p_{n_1}},d_{q_1},\cdots,d_{q_{n_2}}}\\
	= &\sum_{j_{q_1},\cdots,j_{q_{n_2}}}(\sum_{j_{p_1},\cdots,j_{p_{n_1}}}g_{j_{p_1},\cdots,j_{p_{n_1}},j_{q_1},\cdots,j_{q_{n_2}}}U^{(p_1)}_{d_{p_1},j_{p_1}}\cdots U^{(p_{n_1})}_{d_{p_{n_1}},j_{p_{n_1}}})U^{(q_1)}_{d_{q_1},j_{q_1}}\cdots U^{(q_{n_2})}_{d_{q_{n_2}},j_{q_{n_2}}}\\
		=&\sum_{j_{q_1},j_{q_2},\cdots,j_{q_{n_2}}}H_{d_{p_1},d_{p_2},\cdots,d_{p_{n_1}},j_{q_1},j_{q_2},\cdots,j_{q_{n_2}}}(U^{(q_1)}_{d_{q_1},j_{q_1}}\cdots U^{(q_{n_2})}_{d_{q_{n_2}},j_{q_{n_2}}})
	\end{eqnarray*}
	where $H_{d_{p_1},\cdots,d_{p_{n_1}},j_{q_1},\cdots,j_{q_{n_2}}}=\sum_{j_{p_1},\cdots,j_{p_{n_1}}}g_{j_{p_1},\cdots,j_{p_{n_1}},j_{q_1},\cdots,j_{q_{n_2}}}U^{(p_1)}_{d_{p_1},j_{p_1}}\cdots U^{(p_{n_1})}_{d_{p_{n_1}},j_{p_{n_1}}} $.
	Then,
	\begin{eqnarray*}
		[\mathcal{A}]^{(p,q)}&=&[\mathcal{H}]^{(p,q)}(\mathbf{U}^{(q_{n_2})}\otimes \mathbf{U}^{(q_{{n_2}-1})} \otimes \cdots\otimes \mathbf{U}^{(q_1)})^T\\
		&=&(\mathbf{U}^{(p_1)}\otimes \mathbf{U}^{(p_2)} \otimes \cdots\otimes \mathbf{U}^{(p_{n_1})})[\mathcal{G}]^{(p,q)}(\mathbf{U}^{(q_{n_2})}\otimes \mathbf{U}^{(q_{{n_2}-1})} \otimes \cdots\otimes \mathbf{U}^{(q_1)})^T.
	\end{eqnarray*}
\end{proof}

For simplicity, we denote
$$[\mathcal{A}]^{(p,q)}=\mathbf{U}_{od}[\mathcal{G}]^{(p,q)}\mathbf{U}^T_{en},$$
where
$$
\mathbf{U}_{od}= \mathbf{U}^{(p_1)}\otimes \mathbf{U}^{(p_2)} \otimes \cdots\otimes \mathbf{U}^{(p_{n_1})}\in \mathbb{R}^{M^t\times J_{p_1}J_{p_2}\cdots J_{p_{n_1}}},
$$

$$
\mathbf{U}_{en}=\mathbf{U}^{(q_{n_2})}\otimes \mathbf{U}^{(q_{{n_2}-1})} \otimes \cdots\otimes \mathbf{U}^{(q_1)}\in \mathbb{R}^{M^t\times J_{q_1}J_{q_2}\cdots J_{q_{n_2}}},
$$

$$
[\mathcal{G}]^{(p,q)}=g_{j_{p_1},j_{p_2},\cdots,j_{p_{n_1}},j_{q_1},j_{q_2},\cdots,j_{q_{n_2}}}\in \mathbb{R}^{J_{p_1}J_{p_2}\cdots J_{p_{n_1}}\times J_{q_1}J_{q_2}\cdots J_{q_{n_2}}}.
$$
We get
$$
rank(\mathbf{U}_{od})=rank(\mathbf{U}^{(p_1)})\cdots rank(\mathbf{U}^{(p_{n_1})})=J_{p_1}J_{p_2}\cdots J_{p_{n_1}},
$$
$$
rank(\mathbf{U}_{en})=rank(\mathbf{U}^{(q_1)})\cdots rank(\mathbf{U}^{(q_{n_2})})=J_{q_1}J_{q_2}\cdots J_{q_{n_2}}.
$$
\begin{lemma}
	If each factor matrix of tensor $\mathcal{A}$ has full column rank, i.e., $\mathbf{U}^{(1)},\cdots,\mathbf{U}^{(N)}$ has full column rank, then $rank([\mathcal{A}]^{(p,q)})=rank([\mathcal{G}]^{(p,q)})$.
\end{lemma}
\begin{proof}
	$rank([\mathcal{A}]^{(p,q)})=rank(\mathbf{U}_{od}[\mathcal{G}]^{(p,q)}\mathbf{U}^T_{en})=rank([\mathcal{G}]^{(p,q)}\mathbf{U}^T_{en})=rank([\mathcal{G}]^{(p,q)})$.
\end{proof}

\textbf{Proof of Theorem 1}
\begin{proof}
	According to Lemma 1, it suffices to prove that the rank of $[\mathcal{A}^y]^{(p,q)}$ is at least $r$ almost everywhere. From Lemma 3, equivalently, we prove the rank of $[\mathcal{G}]^{(p,q)}$ is at least $r$ almost everywhere.
	
	For any $x\in \mathbb{R}^{J_{p_1}J_{p_2}\cdots J_{p_{n_1}}J_{q_1}J_{q_2}\cdots J_{q_{n_2}}}$, and all possible subsets $\mathbf{p}=\{p_1,p_2,\cdots,p_{n_1}\}$ and the corresponding compliment set $\mathbf{q}=\{q_1,q_2,\cdots, q_{ n_2}\}$, $n_1+n_2=N$. We let $[\mathcal{G}]^{(p,q)}(x)\in \mathbb{R}^{J_{p_1}J_{p_2}\cdots J_{p_{n_1}}\times J_{q_1}J_{q_2}\cdots J_{q_{n_2}}} $, which simply holds the elements of $x$.
Because $r=\max_{p,q}[\mathcal{G}]^{(p,q)}$ for all possible subsets $\mathbf{p},\mathbf{q}$. For all $x$, we have $rank([\mathcal{G}]^{(p,q)}(x))\leq r$. In  the following, we will prove that the Lebesgue measure of the space that $rank([\mathcal{G}]^{(p,q)}(x))<r$ is zero.
	
	Let $([\mathcal{G}]^{(p,q)}(x))_r$ be the top-left $r\times r$ sub matrix of $[\mathcal{G}]^{(p,q)}(x)$ and $det(([\mathcal{G}]^{(p,q)}(x))_r)$ is the determinant, as we know that $det(([\mathcal{G}]^{(p,q)}(x))_r)$ is a polynomial in the entries of $x$, according to theorem in\cite{caron2005zero}, it either vanishes on a set of zero measure or it is the zero polynomial. It implies that the Lebesgue measure of the space whose $det(([\mathcal{G}]^{(p,q)}(x))_r)=0$
	is zero, i.e., the Lebesgue measure of the space whose rank less than $r$ is zero. The result thus follows.
\end{proof}

\subsection{Proof of Theorem 2}
In this subsection, we will prove Theorem 2, the connection of Tucker tensor format and $2^L$ hierarchical Tucker tensor format. The expressive power of hierarchical Tucker tensor network has been well discussed in  \cite{cohen2016expressive}.

In Section 2, we defined  $(p,q)$-matricization which is a kind of general matricization. In the following, we simply consider the proper order matricization of tensor,  denoted as the matrix $[\mathcal{A}]_{p}$ here, for example, for $N=2t$, $[\mathcal{A}]_{p}\in \mathbb{R}^{M_1\cdots M_t \times M_{t+1}\cdots M_{2t}}$; for $N=2t+1$, $[\mathcal{A}]_{p}\in \mathbb{R}^{M_1\cdots M_t \times M_{t+1}\cdots M_{2t+1}}$.

The  $2^L$ hierarchical tensor decomposition format is given as follows:
\begin{eqnarray}\label{HT}
\mathcal{A} & = & \sum^{r_{L-1}}_{\alpha=1}a^{L}_\alpha \phi^{L-1,1,\alpha}\circ \phi^{L-1,2,\alpha}, \\
\phi^{l,t_l,\gamma_l} & = & \sum^{r_{l-1}}_{\alpha=1}a^{l,t_l,\gamma_l}_\alpha \phi^{l-1,2t_l-1,\alpha}\circ \phi^{l-1,2t_l,\alpha},
\quad t_l\in [2^{L-l}],\gamma_l\in[r_{l}],  l=2,\cdots,L-1. \nonumber\\
\phi^{1,t_1,\gamma_1} & = & \sum^{r_0}_{\alpha=1}a^{1,t_1,\gamma_1}_\alpha \mathbf{u}^{0,2t_1-1,\alpha}  \circ\mathbf{u}^{0,2t_1,\alpha},
\quad t_1\in [2^{L-1}],\gamma_1\in[r_{1}].  \nonumber \\
\end{eqnarray}

It is easy to check that $\{\phi^{l,t_l,\gamma_l}\}^{L-1}_{l=1}$ are tensors. The vectorization of these tensors $\{\phi^{l,t_l,\gamma_l}\}^{L-1}_{l=1}$, are denoted as $\{\phi^{l,t_l,\gamma_l}_c\}^{L-1}_{l=1}$ for simplicity. The matrix $[\phi^{l,t_l,\gamma_l}]_{p}$ of tensor $\phi^{l,t_l,\gamma_l}$   has $M_1M_2\cdots M_{2^{l-1}}$ rows and $M_{2^{l-1}+1}M_{2^{l-1}+2}\cdots M_{2^l}$ columns. Note that each dimension of $\mathcal{A}$ is equal to $M$, the matrix $[\phi^{l,t_l,\gamma_l}]_{p}$  contains $M^{2^{l-1}}$ rows and $M^{2^{l-1}}$ columns.
At the $(L-1)$-th layer, define,
\begin{equation}\label{L-1}
\begin{split}
\Phi^{L-1,1}=[\phi^{L-1,1,1}_c, \phi^{L-1,1,2}_c,\cdots,\phi^{L-1,1,\gamma_{L-1}}_c],\\
\Phi^{L-1,2}=[\phi^{L-1,2,1}_c, \phi^{L-1,2,2}_c,\cdots,\phi^{L-1,2,\gamma_{L-1}}_c].
\end{split}
\end{equation}
At the $l$-th layer, $l=2,\cdots, L-2$,
\begin{equation}\label{l}
\begin{split}
\Phi^{l,1}=[\phi^{l,1,1}_c, \phi^{l,1,2}_c,\cdots,\phi^{l,1,\gamma_{l}}_c],\\
\Phi^{l,2}=[\phi^{l,2,1}_c, \phi^{l,2,2}_c,\cdots,\phi^{l,2,\gamma_{l}}_c],\\
\cdots\\
\Phi^{l,2^{L-l}}=[\phi^{l,2^{L-l},1}_c, \phi^{l,2^{L-1},2}_c,\cdots,\phi^{l,2^{L-l},\gamma_{l}}_c].
\end{split}
\end{equation}
At the $1$-st layer,
\begin{equation}\label{1st}
\begin{split}
\Phi^{1,1}=[\phi^{1,1,1}_c, \phi^{1,1,2}_c,\cdots,\phi^{1,1,\gamma_{1}}_c],\\
\Phi^{1,2}=[\phi^{1,2,1}_c, \phi^{1,2,2}_c,\cdots,\phi^{1,2,\gamma_{1}}_c],\\
\cdots\\
\Phi^{1,2^{L-1}}=[\phi^{1,2^{L-1},1}_c, \phi^{1,2^{L-1},2}_c,\cdots,\phi^{1,2^{L-1},\gamma_{1}}_c].
\end{split}
\end{equation}

\textbf{Proof of Theorem 2}
\begin{proof}
	From (\ref{HT}), for the case $l=1$, we have
	$$
	\phi^{1,t_1,\gamma_1}=\sum^{r_0}_{\alpha=1}a^{1,t_1,\gamma_1}_\alpha \mathbf{u}^{0,2t_1-1,\alpha}\circ \mathbf{u}^{0,2t_1,\alpha},~~~t_1\in [2^{L-1}],\gamma_1\in[r_1].
	$$
	Denote $\mathbf{U}^{0,2t_1-1}\in \mathbb{R}^{M\times r_0}$  the matrix with columns $\{ \mathbf{u}^{0,2t_1-1,\alpha}\}^{r_0}_{\alpha=1}$, $\mathbf{U}^{0,2t_1}\in \mathbb{R}^{M\times r_0}$ the matrix with columns $\{ \mathbf{u}^{0,2t_1,\alpha}\}^{r_0}_{\alpha=1}$, and $\mathbf{D}^{1,t_1,\gamma_1}\in \mathbb{R}^{r_0\times r_0}$ the diagonal matrix with $a^{1,t_1,\gamma_1}_\alpha$ on its diagonal. We  rewrite
	$$
[\Phi^{1,1,\gamma_1}]_{p}=\mathbf{U}^{0,1}\mathbf{D}^{1,1,\gamma_1}(\mathbf{U}^{0,2})^T,\cdots,[\Phi^{1,2^{L-1},\gamma_1}]_{p}=\mathbf{U}^{0,2^L-1}\mathbf{D}^{1,2^{L-1},\gamma_1}(\mathbf{U}^{0,2^L})^T,
	$$
	therefore,
	$$
	\Phi^{1,1}=[\mathbf{U}^{0,2}\otimes \mathbf{U}^{0,1}](vec(\mathbf{D}^{1,1,1}),\cdots,vec(\mathbf{D}^{1,1,r_1})),\cdots,$$
	$$
	\Phi^{1,2^{L-1}}=[\mathbf{U}^{0,2^L}\otimes \mathbf{U}^{0,2^L-1}](vec(\mathbf{D}^{1,2^{L-1},1}),\cdots,vec(\mathbf{D}^{1,2^{L-1},r_1})).
	$$
	Here, for simplicity, we denote $(vec(\mathbf{D}^{l,t_l,1}),\cdots,vec(\mathbf{D}^{l,t_l,r_l}))$ as $\mathbf{P}^{l,t_l}$.
	
	For the case $l=2$,
	$$
	[\Phi^{2,1,\gamma_2}]_{p}=\Phi^{1,1}\mathbf{D}^{2,1,\gamma_2}(\Phi^{1,2})^T,\cdots,[\Phi^{2,2^{L-2},\gamma_2}]_{p}=\Phi^{1,2^{L-1}-1}\mathbf{D}^{2,2^{L-2},\gamma_2}(\Phi^{1,2^{L-1}})^T,
	$$
	then,
	$$
	[\Phi^{2,1,\gamma_2}]_{p}=[\mathbf{U}^{0,2}\otimes \mathbf{U}^{0,1}]\mathbf{P}^{1,1} \mathbf{D}^{2,1,\gamma_2}(\mathbf{P}^{1,2})^T[\mathbf{U}^{0,4}\otimes \mathbf{U}^{0,3}]^T,\cdots
	$$
	$$
	[\Phi^{2,2^{L-2},\gamma_2}]_{p}=[\mathbf{U}^{0,2^{L}-2}\otimes \mathbf{U}^{0,2^{L}-3}]\mathbf{P}^{1,2^{L-1}-1} \mathbf{D}^{2,2^{L-2},\gamma_2}(\mathbf{P}^{1,2^{L-1}})^T[\mathbf{U}^{0,2^{L}}\otimes \mathbf{U}^{0,2^{L}-1}]^T,
	$$
	therefore,
	$$
	\Phi^{2,1}=[\mathbf{U}^{0,4}\otimes \mathbf{U}^{0,3}\otimes\mathbf{U}^{0,2}\otimes \mathbf{U}^{0,1}](vec(\tilde{\mathbf{P}}^{2,1,1}),\cdots,vec(\tilde{\mathbf{P}}^{2,1,r_2})),\cdots,
	$$
	$$
	\Phi^{2,2^{L-2}}=[\mathbf{U}^{0,2^{L}}\otimes \mathbf{U}^{0,2^{L}-1}\otimes\mathbf{U}^{0,2^{L}-2}\otimes \mathbf{U}^{0,2^{L}-3}](vec(\tilde{\mathbf{P}}^{2,2^{L-2},1}),\cdots,vec(\tilde{\mathbf{P}}^{2,2^{L-2},r_2})),
	$$
	where $\tilde{\mathbf{P}}^{2,1,\gamma_2}=\mathbf{P}^{1,1} \mathbf{D}^{2,1,\gamma_2}(\mathbf{P}^{1,2})^T,\cdots,\tilde{\mathbf{P}}^{2,2^{L-2},\gamma_2}=\mathbf{P}^{1,2^{L-1}-1} \mathbf{D}^{2,2^{L-2},\gamma_2}(\mathbf{P}^{1,2^{L-1}})^T$.
	
	Assume for the case $l=l$, $\Phi^{l,t}$ always be written as
	$$
	\Phi^{l,1}=[\mathbf{U}^{0,2^{l}}\otimes \mathbf{U}^{0,2^{l}-1}\otimes\cdots\otimes \mathbf{U}^{0,1}][vec(\tilde{\mathbf{P}}^{l,1,1}),\cdots,vec(\tilde{\mathbf{P}}^{l,1,r_l})],
	\cdots,
	$$
	$$
	\Phi^{l,2^{L-l}}=[\mathbf{U}^{0,2^{L}}\otimes \mathbf{U}^{0,2^{L}-1}\otimes\cdots\otimes \mathbf{U}^{0,2^L-2^l+1}][vec(\tilde{\mathbf{P}}^{l,2^{L-l},1}),\cdots,vec(\tilde{\mathbf{P}}^{l,2^{L-l},r_l})].
	$$
	For the case $l=l+1$,
	$$
	[\Phi^{l+1,1,\gamma_{l+1}}]_{p}=\Phi^{l,1}\mathbf{D}^{l+1,1,\gamma_{l+1}}(\Phi^{l,2})^T,\quad \cdots,
    $$
    $$
    \cdots,\quad
	[\Phi^{l+1,2^{L-l-1},\gamma_{l+1}}]_{p}=\Phi^{l,2^{L-l}-1}\mathbf{D}^{l+1,2^{L-l-1},\gamma_{l+1}}(\Phi^{l,2^{L-l}})^T,
	$$
	then
\begin{eqnarray*}
	[\Phi^{l+1,1,\gamma_{l+1}}]_{p}=[\mathbf{U}^{0,2^{l}}\otimes\cdots\otimes \mathbf{U}^{0,1}]\tilde{\mathbf{P}}^{l,1}
\mathbf{D}^{l+1,1,\gamma_{l+1}}(\tilde{\mathbf{P}}^{l,2})^T[\mathbf{U}^{0,2^{l+1}}\otimes \mathbf{U}^{0,2^{l+1}-1}\otimes\cdots\otimes \mathbf{U}^{0,2^l+1}]^T,
\end{eqnarray*}
	\begin{eqnarray*}
\cdots,\quad [\Phi^{l+1,2^{L-l-1},\gamma_{l+1}}]_{p}&=&[\mathbf{U}^{0,2^{L}-2^l}\otimes \mathbf{U}^{0,2^{L}-2^l-1}\otimes\cdots\otimes \mathbf{U}^{0,2^{L}-2^{l+1}+1}]\tilde{\mathbf{P}}^{l,2^{L-l}-1}\\
		&~&\mathbf{D}^{l+1,2^{L-l-1},\gamma_{l+1}}(\tilde{\mathbf{P}}^{l,2^{L-l}})^T[\mathbf{U}^{0,2^{L}}\otimes \mathbf{U}^{0,2^{L}-1}\otimes\cdots\otimes \mathbf{U}^{0,2^L-2^l+1}]^T,
	\end{eqnarray*}
	where $(vec(\tilde{\mathbf{P}}^{l,t_l,1}),\cdots,vec(\tilde{\mathbf{P}}^{l,t_l,r_l}))$ is denoted as $\tilde{\mathbf{P}}^{l,t_l}$.
	
	From the above, we can deduce that
	\begin{eqnarray*}
		[\mathcal{A}]_{p}&=&[\mathbf{U}^{0,2^{L-1}}\otimes \mathbf{U}^{0,2^{L-1}-1}\otimes\cdots\otimes \mathbf{U}^{0,1}]\tilde{\mathbf{P}}^{L-1,1}\\
&~&\mathbf{D}^{L,1}(\tilde{\mathbf{P}}^{L-1,2})^T[\mathbf{U}^{0,2^{L-1}+1}\otimes \mathbf{U}^{0,2^{L-1}+2}\otimes\cdots\otimes \mathbf{U}^{0,2^L}]^T,
	\end{eqnarray*}
	i.e.,
	$$
	\mathcal{A}=\mathcal{G}\times_1 \mathbf{U}^{0,2^{L-1}}\times_2 \mathbf{U}^{0,2^{L-1}-1}\cdots\times_{2^{L-1}}\mathbf{U}^{0,1}(\times_{2^{L-1}+1}\mathbf{U}^{0,2^{L-1}+1}\times \cdots\times_{2^L}\mathbf{U}^{0,2^L}),
	$$
where $[\mathcal{G}]_{p}=\tilde{\mathbf{P}}^{L-1,1}\mathbf{D}^{L,1}
(\tilde{\mathbf{P}}^{L-1,2})^T$.
\end{proof}

\textbf{Proof of Theorem 3}

 From the proof of Theorem 2, we know that, $rank_{Tucker}\mathcal{A}\leq r_0$, here $r_0$ is the rank of $1$-st layer of hierarchical Tucker tensor. Because of the assumption $rank_{HT}\mathcal{A}\leq r$, we get that $r_0\leq r$. $rank_{Tucker}\mathcal{A}\leq r_0\leq r$ established.

\end{document}